\documentclass[letterpaper]{article} 
\usepackage{aaai24}  
\usepackage{times}  
\usepackage{helvet}  
\usepackage{courier}  
\usepackage[hyphens]{url}  
\usepackage{graphicx} 
\usepackage{natbib}  
\usepackage{caption} 
\usepackage{algorithm}
\usepackage{algorithmic}

\usepackage{amsmath,amssymb}
\usepackage{dsfont}
\usepackage{enumitem}
\usepackage{multirow}
\usepackage{amsmath}
\usepackage{amsthm}
\usepackage{svg}

\newtheorem{theorem}{Theorem}

\usepackage{array}
\newcolumntype{C}[1]{>{\centering\let\newline\\\arraybackslash\hspace{0pt}}p{#1}}
\usepackage{newfloat}
\usepackage{listings}
\DeclareCaptionStyle{ruled}{labelfont=normalfont,labelsep=colon,strut=off} 
\floatstyle{ruled}
\newfloat{listing}{tb}{lst}{}
\floatname{listing}{Listing}
\title{i-Rebalance: Personalized Vehicle Repositioning for Supply Demand Balance}
\author{
    Haoyang Chen\textsuperscript{\rm 1},
    Peiyan Sun\textsuperscript{\rm 1},
    Qiyuan Song\textsuperscript{\rm 1},
    Wanyuan Wang\textsuperscript{\rm 1},\\
    Weiwei Wu\textsuperscript{\rm 1},
    Wencan Zhang\textsuperscript{\rm 2},
    Guanyu Gao\textsuperscript{\rm 3},
    Yan Lyu\textsuperscript{\rm 1}\footnote{Corresponding author.}
}

\affiliations {
    \textsuperscript{\rm 1}Southeast Univeristy, China\\
    \textsuperscript{\rm 2}National University of Singapore\\
    \textsuperscript{\rm 3}Nanjing University of Science and Technology\\
    \{hy\_chen, spy, qiyuan\_song, wywang, weiweiwu\}@seu.edu.cn,
    wencanz@u.nus.edu,
    gygao@njust.edu.cn, lvyanly@seu.edu.cn
}

\usepackage{bibentry}

\begin{document}

\maketitle


\begin{abstract}
Ride-hailing platforms have been facing the challenge of balancing demand and supply. Existing vehicle reposition techniques often treat drivers as homogeneous agents and relocate them deterministically, assuming compliance with the reposition.
In this paper, we consider a more realistic and driver-centric scenario where drivers have unique cruising preferences and can decide whether to take the recommendation or not on their own. 
We propose i-Rebalance, a personalized vehicle reposition technique with deep reinforcement learning (DRL). 
i-Rebalance estimates drivers' decisions on accepting reposition recommendations through an on-field user study involving 99 real drivers. 
To optimize supply-demand balance and enhance preference satisfaction simultaneously, i-Rebalance has a sequential reposition strategy with dual DRL agents: Grid Agent to determine the reposition order of idle vehicles, and Vehicle Agent to provide personalized recommendations to each vehicle in the pre-defined order. This sequential learning strategy facilitates more effective policy training within a smaller action space compared to traditional joint-action methods. 
%
%
Evaluation of real-world trajectory data shows that i-Rebalance improves driver acceptance rate by 38.07\% and total driver income by 9.97\%. The code for our approach is available at https://github.com/Haoyang-Chen/i-Rebalance.



\end{abstract}

\pagenumbering{Roman}
\setcounter{page}{1}

\section{Introduction}

Taxi and ride-hailing services have been facing supply-demand imbalances. Thousands of drivers spend almost 50\% of their working time cruising on road for passengers~\cite{zong2018taxi}. Meanwhile, thousands of travel demands got unserved due to limited vehicle supply nearby. 
There have been efforts on reposition idle vehicles to demanding locations for supply-demand balance with deep reinforcement learning (DRL) techniques~\cite{qin2022reinforcement,yang2020multiagent,liu2021meta}.
%
%
%
By interacting with actual travel demands, these methods can adapt to the complex and dynamic environment and coordinate competing drivers.

However, these techniques hypothesize that \textit{all the drivers unquestionably adhere to the reposition recommendation and follow the recommended cruising route.}
This is impractical because every driver has a unique cruising preference~\cite{cGAIL, TrajGAIL}. Imposing repositioning mandates on drivers would lead to unpleasant working experiences and cause driver loss on ride-hailing platforms. 

To this end, we are motivated to consider a more realistic and driver-centric scenario where drivers could freely accept or reject reposition recommendations. This requires generating personalized reposition recommendation 
to improve the probability of acceptance. 
However, the challenges are:


%




\begin{figure}[t]
    \centering
\includegraphics[width=\linewidth]{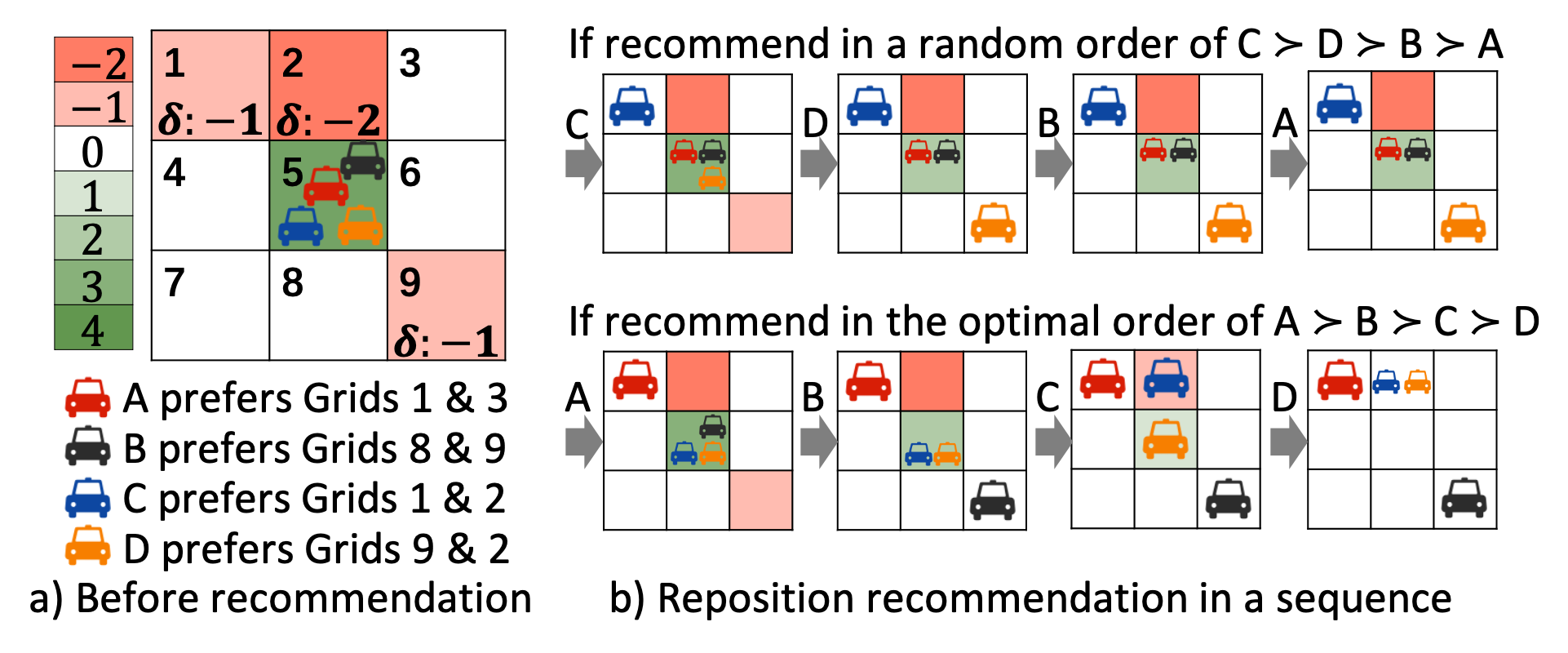}
    \vspace{-0.3in}
    \caption{Impact of recommendation order. a)~Initial scenario: Recommending four idle vehicles in Grid 5. Each prefers two neighboring grids, rejecting non-preferred options to stay at the current grid. The supply-demand gap $\delta$ is color-coded: red signifies shortage, and green indicates excess idle vehicles. 
    b)~Two recommendation orders $C\succ D\succ B \succ A$ (top row) and $A\succ B\succ C\succ D$ (bottom row) lead to different supply-demand balances due to drivers’ diverse preferences.
    }
    \label{Fig:sequenceExample}
    \vspace{-0.2in}
\end{figure}

1)~\textit{Subjective driver acceptance to reposition recommendation}. 
Drivers can decide whether to accept recommendations in most ride-hailing platforms on their own. However, there has been lacking studies on under which circumstances and how likely the driver would accept the reposition recommendation. 
Such decision-making from real drivers should be collected for accurate decision modeling. 

2)~\textit{Personalized repositioning}.  
%
To incorporate unique preferences of drivers, we have to jointly decide each individual driver to visit a specific location. However, learning such a joint action can be challenging due to the large action space. For a smaller action space, we could reposition vehicles one by one, however, \textit{the reposition order of vehicles impacts the resulting supply-demand balance}, due to the fact that drivers with diverse preferences may reject unsatisfactory recommendation. Take Figure 1 as an example of recommending four vehicles in Grid 5 to address the supply-demand gap in neighboring grids. Each vehicle prefers two neighboring grids, rejecting non-preferred options to stay at the current grid. If we recommend in the order of $C\succ D\succ B\succ A$ (top row), drivers C and D will visit more demanding Grids 1 and 9 based on their preference. But the subsequent drivers B and A will reject recommendations to address the demand at Grid 2, leading to an imbalanced supply-demand. Conversely, the order $A\succ B\succ C\succ D$ (bottom row) achieves a more favorable outcome. All drivers visit preferred grids, simultaneously mitigating the supply-demand gap. Therefore, it is necessary to seek for such an order to help with better supply-demand balance and preference satisfaction. 



In this paper, we propose i-Rebalance, a personalized vehicle repositioning technique to balance supply and demand. To estimate drivers’ subjective decisions on reposition recommendation, we recruited 99 real drivers and conducted an on-field user study to collect their decision makings. For personalized repositioning, i-Rebalance sequentially repositions vehicles with dual DRL agents: Grid Agent and Vehicle Agent. Grid Agent learns \textit{the repositioning order of idle vehicles} within a grid, while Vehicle Agent learns \textit{personalized repositioning to each vehicle} in the pre-defined order. By joint training of the two agents, i-Rebalance can simultaneously optimize supply-demand balance and driver acceptance, facilitating policy training with a small action space. In summary, our contributions are: 

%
\begin{itemize}
    \item A personalized vehicle reposition technique i-Rebalance to balance supply and demand with DRL. It improves cruising preference satisfaction of drivers and their acceptance of repositions.
    
    
    \item A user study on 99 recruited ride-hailing drivers to estimate the likelihood of reposition acceptance. 

    \item 
    A sequential vehicle reposition framework with dual DRL agents to determine the reposition order and then the personalized reposition for each vehicle by the order.

    
    \item Evaluation on real-world taxi trajectory data shows that i-Rebalance greatly improves reposition acceptance of drivers and total driver income compared to baselines. 
    
\end{itemize}

\section{Related Work}

Recently, deep reinforcement learning (DRL) has been widely used in vehicle repositioning problems~\cite{qin2022reinforcement,liu2022deep,jiao2021real}. 
By interacting with a simulated ride-hailing system, DRL techniques are able to capture the interactions between repositioning actions and dynamic demands and can coordinate competing vehicles. Centralized strategies have been explored, where a single agent is trained to govern reposition for all vehicles~\cite{he2020spatio,oda2018movi,liu2020context}. However, such approaches suffer from extensive joint action search spaces. In contrast, decentralized approaches consider each vehicle as an independent agent~\cite{al2019deeppool,wen2017rebalancing}, but face the non-stationary exploration problems in multiagent RL tasks~\cite{gao2022distributed}.
%
%


To bridge traditional vehicle repositioning methods ~\cite{8105835,xie2018privatehunt,xu2018taxi} and DRL, some integrate request-matching and vehicle-repositioning into a single reinforcement learning framework~\cite{xu2023multi}, while others leverage hierarchical reinforcement learning to coordinate vehicles~\cite{jin2019coride,xi2022hmdrl}. Further efforts focus on simplifying the model, such as treating multiple agents in the same location as homogeneous agent to share the same policy network~\cite{lin2018efficient}; averaging actions of agents in the same grids~\cite{li2019efficient}; and combining vehicle-location maximum matching optimization during DRL training~\cite{liu2021meta}.
%
%


%
However, these methods assume each driver complies with the reposition recommendations unconditionally.
In practice, drivers have their own preference on cruising routes~\cite{cGAIL,TrajGAIL} and can reject reposition recommendations in most ride-hailing platforms. Recently, \cite{he2020spatio} includes collective driver preferences in their repositioning system that helps recommend locations that most drivers like to visit, without personalization.
\cite{xu2020recommender} integrates driver acceptance into its reposition model, but also overlooks personalized driver preferences.
In contrast, i-Rebalance is the first to consider drivers' personalized cruising preferences and their freedom of acceptance of reposition.
Specifically, we predict driver preferences with a LSTM network and conduct a user study to estimate the likelihood of acceptance. We propose a novel sequential vehicle reposition framework to learn reposition order of idle vehicles and to learn personalized reposition for each vehicle in the pre-defined order, respectively.

%

\begin{figure*}[!ht]
    \centering
    \includegraphics[width=0.97\textwidth]{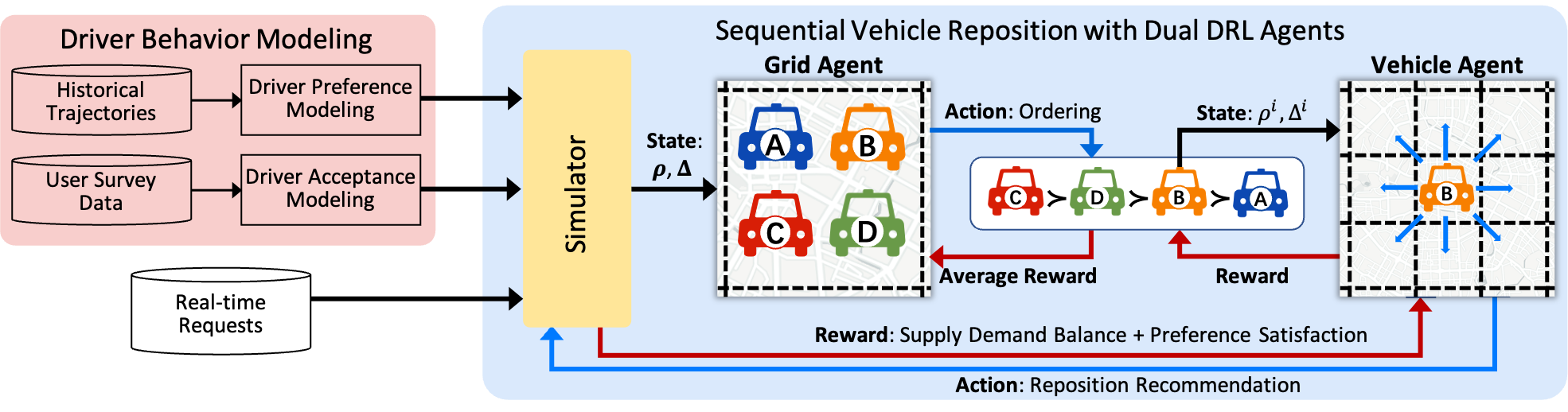}
    \vspace{-0.1in}
    \caption{Overview of i-Rebalance. 
    i-Rebalance comprises two phases: 1)~\textit{Driver Behavior Modeling} simulates realistic driver decision-making by predicting their cruising preferences and reposition acceptance probabilities. 2)~\textit{Sequential Vehicle Reposition with Dual DRL Agents} interacts with the simulator. Grid Agent observes nearby supply-demand gap $\Delta$ and driver preference $\mathbf{\rho}$ and determines the repositioning order of idle vehicles within the grid, e.g., $C\succ D \succ B \succ A$. By this order, Vehicle Agent observes individual preference $\rho^i$ of driver $i$ and real-time updated supply-demand gap $\Delta^i$, and recommends reposition destinations for this vehicle. It receives rewards of supply-demand balance and preference satisfaction for each recommendation, while Grid Agent receives the average rewards after all recommendations.
    }
    \label{Fig:Overview}
    \vspace{-0.1in}
\end{figure*}

\section{Problem Statement}

We split the entire city into nonoverlapping rectangle grid map~\cite{he2020spatio}, and time of a day into equal-length time intervals (e.g., 10 minutes). 
Based on this setting, we formally define the following concepts and the problem of personalized reposition recommendation.

\textbf{Driver Cruising Preference}. 
Each taxi driver has a unique preference for cruising routes for passengers. For example, some may prefer to hang around hot locations such as train stations and hospitals, some may prefer familiar locations near home. 
We define the cruising preference of driver $i$ located in grid $g$ at time $t$, denoted by $\rho_t^i$, by the probability of how likely the driver is willing to visit the neighboring $3\times 3$ grids (including the current grid $g$). This will be estimated from the driver's historical cruising trajectories. 

\textbf{Reposition Recommendation}. 
Ride-hailing platforms often provide recommendations on where to hunt for passengers to balance supply and demand. 
Instead of recommending an arbitrary location~\cite{liu2022deep} on the map, we define a reposition recommendation at time $t$ as visiting one of the neighboring $3\times 3$ grids (including the current grid $g$) for smaller decision-making space~\cite{liu2020context}.

\textbf{Driver Acceptance} towards reposition recommendation. We define a driver's acceptance as the probability of accepting the recommendation, given the driver's own cruising preference. This will be estimated by our on-field user study.


\textbf{Supply-Demand Gap}. For each grid $g$ at time $t$, we define supply-demand gap, denoted by $\delta_t^g$, as the difference between the number of idle vehicles and the number of ride requests within grid $g$ at time $t$.

\textbf{Problem Statement}. i-Rebalance aims to generate a reposition recommendation for each idle vehicle at each time step $t$ to minimize the overall absolute supply-demand gap and maximize the driver acceptance at the same time. 

Next, we present the overview of i-Rebalance, followed by two driver behavior models to predict driver preferences and their acceptance of recommendations, and a sequential vehicle reposition framework with dual DRL agents.

\section{System Overview}

i-Rebalance comprises two distinct phases: 1)~\textit{Driver Behavior Modeling} for realistic driver decision-making simulation, and 2)~\textit{Sequential Vehicle Reposition with Dual DRL Agents} to learn personalized recommendation via interacting with the simulated drivers and environment (see Figure~\ref{Fig:Overview}).

\textit{Driver Behavior Modeling} predicts each driver's cruising preferences based on their own historical trajectories, and estimates the probability of a driver accepting a reposition recommendation through an on-field user study. The two models are then integrated into the simulator, enabling us to simulate drivers' cruising routes and their likelihood of accepting a reposition recommendation.

\textit{Sequential Vehicle Reposition with Dual DRL Agents} interacts with the simulator to optimize vehicle repositioning across grids.
%
%
To provide personalized reposition, we have to decide the assignment of each individual driver to a specific location. Training for such a joint-action faces challenges of a vast action space. A potential solution is to reposition vehicles one by one in a smaller action space. However, this introduces a new challenge: the repositioning order significantly impacts supply-demand balance. This is because drivers have diverse preferences and may reject unsatisfactory recommendation. Earlier driver decisions affect subsequent ones and thereby the resulting supply-demand balance (as illustrated in Figure~\ref{Fig:sequenceExample}).

%

%
Motivate by this, we seek for the optimal repositioning order that can help with better supply-demand balance and preference satisfaction. Specifically, we employ a Grid Agent to determine the reposition order first. By this order, a Vehicle Agent provides personalized recommendation for each vehicle. Grid Agent observes nearby supply-demand gap and cruising preference of drivers within the grid while Vehicle Agent focuses on each vehicle’s own preference. Vehicle Agent receives intermediate reward for supply-demand balance and preference satisfaction, Grid Agent receives the average of Vehicle Agent's reward for using its order. By training the two agents jointly, both supply-demand balance and preference satisfaction can be optimized simultaneously.

\section{Driver Behavior Modeling}\label{sec:Preference}

\subsection{Driver Cruising Preference}

We propose a lightweight LSTM based network, to predict driver's next cruising directions based on individual past cruising trajectories. The prediction will be used to 1)~estimate driver acceptance to a recommendation, and 2)~simulate preferred cruising directions upon rejection.

As illustrated in Figure~\ref{Fig:PreferenceModeling}, the two-layer LSTM network takes the input of factors including i)~the driver's own working states, such as how long she has been cruising and how much she has earned, and ii)~the current location, which can be presented by distances to a set of POIs, traffics, supply and demand nearby. It predicts probabilities $\rho_t$ of the driver willing to visit the neighboring $3 \times 3$ grids.
See the detailed feature engineering and network description in Appendix A.
This lightweight network improves predictions for drivers with limited historical cruising data, and it's more efficient than larger neural networks~\cite{TrajGAIL,cGAIL}. However, existing trajectory prediction networks are also applicable.

\begin{figure}[t]
    \centering
    \includegraphics[width=\linewidth]{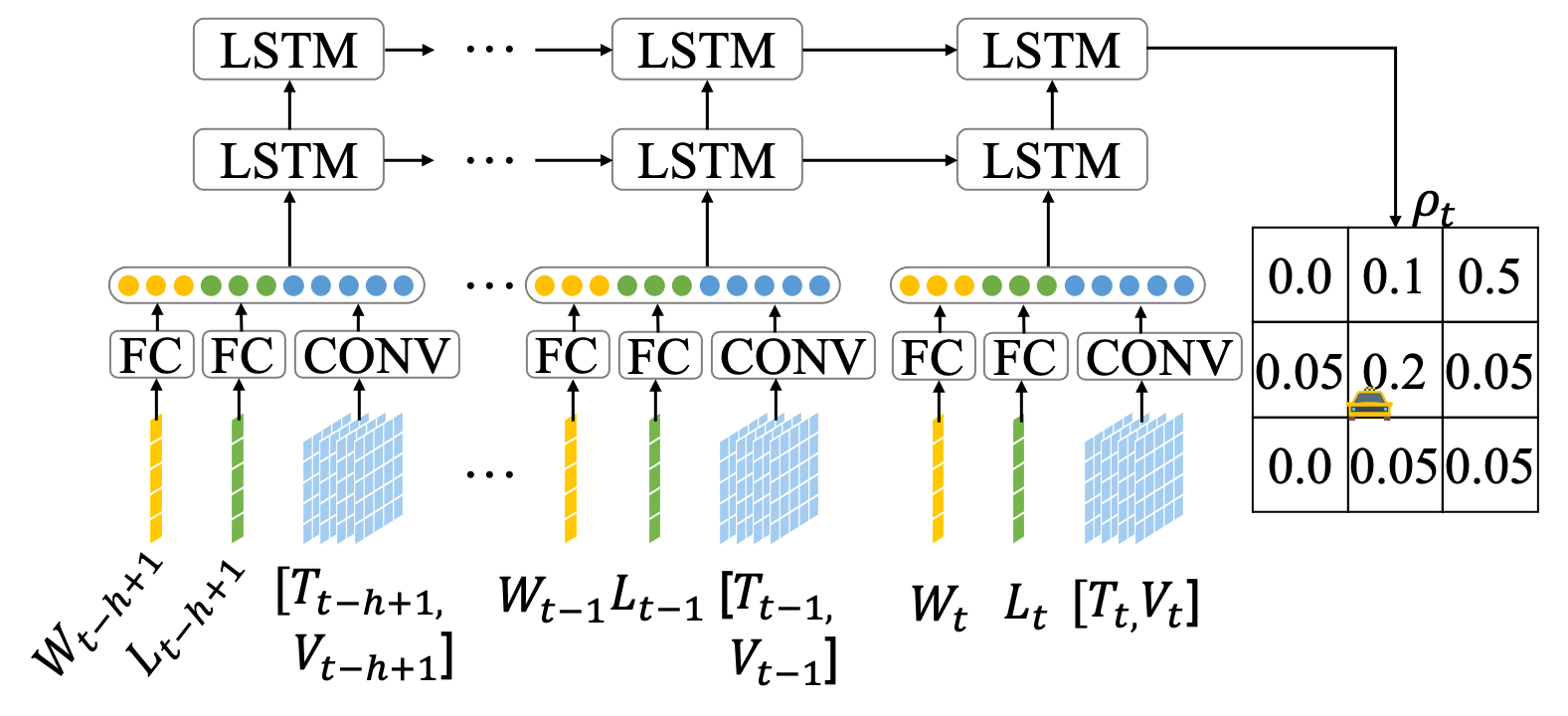}
    \vspace{-0.2in}
    \caption{Cruising preference prediction network. The network takes a sequence of the features including driver states $W_t$, POI distances $L_t$, traffics $T_t$, and supply-demand features $V_t$ from time $t-h+1$ to $t$ as input and predicts probabilities $\rho_t$ of the driver visiting the neighboring $3 \times 3$ grids.}
    \label{Fig:PreferenceModeling}
    \vspace{-0.15in}
\end{figure}

\subsection{Driver Acceptance to Reposition Recommendation}


To estimate the likelihood of a driver accepting a reposition recommendation, we conducted an on-field user study to collect drivers' decisions in various working scenarios. 

We recruited drivers from the Didi platform by requesting rides at a few selected origin locations (including business centers, hospitals, universities, and scenic spots) to a random destination in our city.
%
After a driver comes for the pick-up, she is asked to park her vehicle at a safe place (mostly at the origin) for an interview. 
Drivers followed the procedure: 
1)~The driver answers demographic questions such as working experience and income per ride request. She is also asked to estimate how often she accepts the reposition instructions from the platform (Obedience).
2)~She proceeds to the main survey with 3 trials of different maps with different decision-making questions. 
For each trial, she sees a map with $5 \times 5$ grids of a region in the city and is asked whether she is familiar with the region. She is then asked to imagine herself located in the center grid with an empty vehicle and to rank the neighboring 9 grids (including the center grid) as the next preferred cruising direction. 
We then provide a repositioning recommendation, i.e., visit one of the 9 ranked grids (Preference Ranking), as well as Expected Income if taking the recommendation, and ask the driver to decide whether to accept the recommendation or not. 
For every trial, each driver answers 10 decision-making questions, each with a randomly selected ranked grid as a recommendation and one of 7 expected incomes ranging from 6 to 16.  
%


We recruited 99 Didi drivers at 11 locations, with working experience 0.1-8 (Mean = 3.2) years, and reported income of 10 - 20 (Mean = 14.6) RMB per ride. 
We excluded 6/105 participants who did not complete all questions or gave unreliable answers.
Participants completed the survey questions within 20 minutes and were compensated 12 RMB (\$1.50 USD).
We formulated whether a driver would accept or reject our reposition recommendation as a binary classification problem. We fit a binomial logistic regression model with \textit{Preference Ranking} (denoted as $r$), \textit{Expected Income} (denoted as $m$), and \textit{Driver's Obedience} (denoted as $o$) as independent variables: 
\begin{equation}
P_{accept} = 1 / (1+\exp(-(b + w_rr + w_mm + w_oo))).
\label{eq:logitModel1}
\end{equation}
%

\begin{table}[ht]
\centering
\vspace{-0.1in}
\footnotesize
\begin{tabular}{r@{}C{1.2cm}@{}@{}C{1.0cm}@{}@{}C{1.0cm}@{}@{}C{1.2cm}@{}@{}C{0.8cm}@{}}
\hline
Variable            & Coef. & SE        & $\chi^2$ & p                & R2 \\ 
\hline
Intercept           & -1.31 & 0.26      & 25.06     & \textless{}.0001 & \multirow{4}{*}{.3794} \\
Preference Ranking  & -0.44 & 0.03      & 273.50    & \textless{}.0001 &    \\
Expected Income     & 0.29  & 0.02      & 172.75    & \textless{}.0001 &    \\
Driver's Obedience  & 2.17  & 0.22      & 95.93     & \textless{}.0001 &    \\
\hline
\end{tabular}
\caption{Statistical analysis of driver responses on the binomial logistic regression model predicting the likelihood to accept a repositioning recommendation.}
\label{Table:LogitRegression}
\end{table}

We found that all three factors significantly influenced drivers' willingness to accept the recommendation (Table~\ref{Table:LogitRegression}). The model achieved a classification accuracy of 76.37\% and an area under the ROC curve (AUC) of 0.8266. 
According to the partial dependence plot (Figure \ref{Fig:PDP}), we also observed that drivers have a higher chance of rejecting the less preferred recommendation (lower ranked).  
This indicates that they are more willing to collaborate with the system if it can provide highly personalized recommendations. 
Conversely, drivers are less likely to reject the recommendation if the recommendation helps to bring a higher expected income.
%
Also, a driver who is inherently more obedient exhibits a higher likelihood of accepting the recommendation. 

These statistical results and analysis indicate a good fit of the binomial logistic regression model. We use this model (Eq.~\ref{eq:logitModel1}) with $b = -1.31, w_r = -0.44, w_m = 0.29, w_o = 2.17$ (see  Table~\ref{Table:LogitRegression}) to simulate driver acceptance to a reposition recommendation later. 



\begin{figure}[t]
    \centering
    \includegraphics[width=0.9\linewidth]{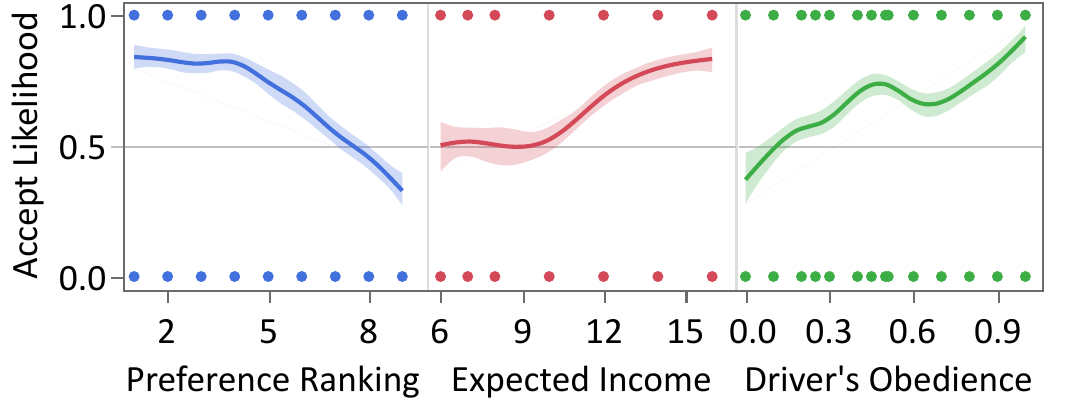}
    \vspace{-0.1in}
    \caption{Partial dependence plot of independent variables.}
    \label{Fig:PDP}
    \vspace{-0.15in}
\end{figure}



\section{Sequential Reposition with Dual DRL Agents}

In this section, we present a sequential vehicle reposition framework with dual DRL agents to learn personalized reposition recommendations.

\subsection{Order Determination by Grid Agent}

We define a \textbf{Grid Agent} to learn the optimal order for recommending idle vehicles within a grid. 
%
%
Its state, action, and reward are defined as follows:

\textbf{State} of Grid Agent, denoted by $S_G$, includes the predicted supply-demand gap in neighboring grids at time step $t+1$, and preferences of available drivers in this grid, i.e.,
%
\begin{equation}
    S_G=\langle \Delta_{t+1},\rho_t^1,\rho_t^2,\dots,\rho_t^n \rangle, 
\end{equation}
where $\Delta_{t+1}$ denotes a vector of supply-demand gaps of neighboring $3\times3$ grids (including the current grid) at time $t+1$, i.e., $\Delta_{t+1} = (\delta_{t+1}^1, \delta_{t+1}^2, \cdots, \delta_{t+1}^9)$. 
The supply-demand gap $\delta_{t+1}^k$ of grid $k$ is estimated from the number of idle vehicles (supply) that will stay or arrive at the grid, subtracted by the number of requests (demand) at time $t+1$  predicted by a recurrent graph convolution network~\cite{wang2020traffic}. 
$\rho_t^i$ denotes the preference (i.e., the probability of visiting each of the 9 grids) of $i$-th driver, $1 \leq i \leq n$. 
Note that the number of available drivers $n$ in a grid is dynamic. 
To ensure dimension consistency for state features, we employ a preference matrix capable of accommodating a larger driver count with non-driver entities padded with zeros.

\textbf{Action} of Grid Agent, denoted by $A_G$, is a vector of priority scores assigned to each driver $i$, i.e., $ A_G = (a_1, a_2, \cdots, a_n)$, where  $0 \leq a_i \leq 1$, $1 \leq i \leq n$.  
%
%
Later on, Vehicle Agent will recommend the drivers based on the order $\succ=(\succ_{1},\cdots,\succ_{n})$ ranked by their priority scores.
%

\textbf{Reward} of Grid Agent, denoted by $R_G$, is the average reward received by Vehicle Agent after all the idle vehicles get recommendations and repositioned by the output order.

\subsection{Personalized Recommendation by Vehicle Agent}

We define a \textbf{Vehicle Agent} to learn personalized recommendations for repositioning idle vehicles within a grid, following the order determined by the Grid Agent.
%
%
The definitions of state, action, and reward are detailed below.

\textbf{State} of a vehicle agent for recommending to driver $i$, denoted by $S_V^i$, includes the real-time updated supply-demand gap in future time $t+1$, denoted by $\Delta_{t+1}^i$,  along with the cruising preference of driver $i$, denoted by $\rho_t^i$, i.e.,
\begin{equation}
    S_V^i = \langle \Delta_{t+1}^i,\rho_t^i \rangle,
\end{equation}
$\Delta_{t+1}^i$ is calculated from the initial supply-demand gap $\Delta_{t+1}$ before recommending for the grid and the reposition decisions made by drivers leading up to the $i$-th driver's turn.
%
$\rho_t^i$ denotes the preference of driver $i$.

\textbf{Action} of a vehicle agent, denoted by $a_V^i$, is to recommend driver $i$ to visit one of the 9 grids $g$. 
%
%
Different from deterministic repositioning methods,
our action simply provides a reposition recommendation, whether the driver accepts the recommendation or insists on her own cruising preferences is determined by herself.  

\textbf{Reward} of a recommendation action should be able to measure supply-demand balance, as well as preference satisfaction to the recommendation. Specifically, 

\textit{Balance Reward}, denoted by $R_B^i(S_V^i,a_V^i)$, is a standardized supply-demand gap at the recommended grid $a_V^i$, i.e., 
\begin{equation}
    R_B^i(S_V^i,a_V^i)= - (\delta_{t+1}^{a_V^i}- \mu_g)/ \sigma_g,
\end{equation}
Here, $\mu_g$ and $\sigma_g$ denote the mean and standard deviation of the supply-demand gap of the neighboring 9 grids of the grid $g$, where $g$ is the current location of driver $i$. This reward guides drivers to visit the most demanding grids.

\textit{Preference Satisfaction Reward}, denoted by $R_P^i(S_V^i,a_V^i)$, measures how much a reposition recommendation satisfies driver's preference. 
It is the min-max normalized rankings of driver $i$'s preferences on the 9 neighboring grids.
%
%
A higher-ranking recommendation gets a higher preference reward.

Overall, the total reward of a reposition recommendation is the summation of the two types of rewards, i.e., 
\begin{equation}
    R_V^i (S_V^i,a_V^i)=  \alpha_B R_B^i(S_V^i,a_V^i) + \alpha_P R_P^i(S_V^i,a_V^i),
    \label{Eq:totalRewards}
\end{equation}
where $\alpha_B$ and $\alpha_P$ are weights of corresponding rewards.

\subsection{Reposition Algorithm} 

i-Rebalance adopts an A2C network to learn personalized reposition policies. 
Algorithm~\ref{algorithm HRLDS} depicts general procedures of training i-Rebalance.
We define an episode as a day and divide the daytime into equal-length time steps. 

The algorithm initializes Grid Agent network $\theta_G$ and Vehicle Agent network $\theta_V$ with vacant memory buffers $M_G$ and $M_V$, respectively. 
At time step $t$, we predict driver preference by our LSTM network (Fig.~\ref{Fig:PreferenceModeling}) and the demand in each grid at time $t+1$. 
Then, we start to reposition idle vehicles in each grid in a radial order starting from the center grid and spiraling outward (Lines~\ref{Line:trainingStart} to \ref{Line:trainingEnd}).
The Grid Agent $\theta_G$ observes the state $S_G$ for a grid, generates idle vehicle order $A_G$, and stores this experience $[S_G, A_G]$ into memory $M_G$.
Vehicle Agent then generates a reposition recommendation $A_V$ for each idle vehicle order $A_G$.
The simulator emulates driver's decision on whether to accept the recommendation or visit the preferred grid by our binomial logistic regression model (Eq.~\ref{eq:logitModel1}). 
Vehicle Agent gets immediate Reward $R_V^i$ after recommending vehicle $i$. The transition $[S_V,A_V,R_V^i]$ is stored to replay buffer $M_V$.
Once all idle vehicles in the grid are repositioned, Grid Agent gets the average reward $R_G$ of all those immediate rewards.
After all idle vehicles are processed, our simulator updates demands and vehicle supplies by repositions outcome and the drop-offs of occupied vehicles (Line~\ref{line:demandupdate}) for time $t+1$. The demands are then matched to their nearest idle vehicles.
Finally, replay buffers are updated by the next state and both Grid Agent and Vehicle Agent are updated by the newest buffers separately.

\begin{algorithm}[tb]
    \caption{\footnotesize Personalized Taxi Repositioning}\label{algorithm HRLDS}
    \label{alg:algorithm}
    \footnotesize
    \textbf{Input}: Max episodes $N$, total time step $T$\\
    \textbf{Output}: Grid Agent $\theta_G$, Vehicle Agent $\theta_V$ \\
    \vspace{-\baselineskip}
    \begin{algorithmic}[1] 
        \STATE Memory $M_V, M_G\leftarrow \emptyset$, randomly initialize network $\theta_V, \theta_G$.
        \WHILE{episode $n<N$}
            \STATE Simulator reset
            \WHILE{$t<T$}
                \STATE Predict demands at $t+1$ and driver preference
                    \FOR{each grid in radial sequence} \label{Line:trainingStart}                       
                        \STATE Get state $S_G$ for grid agent
                        \STATE Generate a repositioning order $A_G$ by network $\theta_G$
                        \STATE Store $[S_G,A_G]$ to $M_G$
                        
                        \FOR{each available driver following $A_G$}
                        \STATE Get state $S_V$ for vehicle agent
                        \STATE Recommend a reposition $A_V$ by network $\theta_V$ 
                        \STATE Driver decides to execute $A_V$ or visit preferred grid
                        \STATE Get rewards $R_B$ and $R_P$
                        \STATE Store $[S_V,A_V,R_P+R_B]$ to $M_V$
                        \ENDFOR
                    \STATE Get average reward $R_G$ and update to $M_G$
                \ENDFOR \label{Line:trainingEnd}
                \STATE Updates real demand and supply\label{line:demandupdate}
                \STATE Match demands to their nearest idle vehicles
                \STATE Update next state $S_V'$ to $M_V$ for repositioned vehicles
                \STATE Replay from buffer $M_V$ to update $\theta_V$
                \STATE Replay from buffer $M_G$ to update $\theta_G$ 
            \ENDWHILE
        \ENDWHILE
    \end{algorithmic}
    
\end{algorithm}



We further prove that our sequential reposition framework 
can achieve the same total rewards as learning joint action of repositioning all idle vehicles simultaneously in a grid.

\begin{theorem}
    Given the state at the grid $g$, $S_G=\langle \Delta_{t+1},\rho_t^1,\rho_t^2,\dots,\rho_t^n\rangle$, the proposed sequential reposition  framework learns reposition order and reposition policy can achieve the same reward with that of the optimal policy of assigning vehicles to locations simultaneously.
\end{theorem}

\begin{proof}
At the state $S_G$, assume that there are $n$ idle vehicles being repositioned and the optimal joint reposition action $\vec{a}^{*}=\langle a_1^*, a_2^*,\cdots, a_n^*\rangle$, which has the expected immediate global reward $\mathbb{E}_{P_{accept}}[\sum_{r\in Neg(g)}\alpha_{B} R_B^{r}+\sum_{i=1}^{n}\alpha_{P} R_{P}^{i}]$, where $Neg(g)$ denote the neighboring grids of $g$ and $R_B^{r}=- (\delta_{t+1}^{r}- \mu_g)/ \sigma_g$ denotes the balance reward of the neighboring grid $r$. Given the joint action $\vec{a}^{*}$, and a randomized order $\succ=(\succ_{1},\cdots,\succ_{n})$, then the same immediate expected global reward $\mathbb{E}_{P_{accept}}[\sum_{r=1}^{9}\alpha_{B} R_B^{r}+\sum_{i=1}^{n}\alpha_{P} R_{P}^{i}]$ will be achieved by the Grid agent if each vehicle $i$ is repositioned by the action $a_i^{*}$ in sequence according to the order $\succ$. In our sequential reposition mechanism, the reposition action of the latter vehicle agent depends on the former vehicle agent's action. Thus, given the $(j-1)$th vehicle's action, if the $j$th vehicle agent would like to change his action from $a_{j}^{*}$ to $a_{j}'$, a larger individual reward must be achieved, i.e.,  
\begin{align} \notag
R_{V}^{j}(S_V^{j},a_{j}')\hspace{-3pt}=\hspace{-3pt}\alpha_{B}R_{B}^{i}(S_V^{j},a_{j}')\hspace{-3pt}+\hspace{-3pt}\alpha_{P}R_{P}^{j}(S_V^{j},a_{j}')\hspace{-3pt}>\hspace{-3pt} R_V^{j}(S_V^{j},a_{j}^{*}).
\end{align}
Because the immediate global reward received by the joint action (i.e., $\mathbb{E}_{P_{accept}}[\sum_{r=1}^{9}\alpha_{B} R_B^{r}+\sum_{i=1}^{n}\alpha_{P} R_{P}^{i}]$) is the sum of the individual vehicle agent's reward, 
improving the action from $a_{j}^{*}$ to $a_{j}'$ also improves the immediate global reward.
%
%
In summary, for any joint reposition policy $\vec{a}^{*}=\langle a_1^*, a_2^*,\cdots, a_n^*\rangle$, our sequential reposition strategy can also find a reposition order and reposition actions that produce immediate global reward not less than that of $\vec{a}^{*}$.   
\end{proof}


\section{Evaluation}
In this section, we evaluate the performance of i-Rebalance in a simulator built with a real-world taxi trajectory dataset. In the following, we first describe experimental settings and then discuss experimental results.  

\subsection{Experimental Settings}

\subsubsection{Environment Simulator}


We employ and extend an open-source simulator~\cite{liu2020context} that simulates ride-hailing platform operations to receive ride requests and dispatch the requests to the nearest idle vehicles. The ride requests are sampled from a real-world taxi trajectory dataset, and driver starting times and locations are initialized accordingly. 
The simulator plans routes with Dijkstra shortest path algorithm and calculates travel time by estimating speed from the dataset for more realistic simulation.

We divided the city map into 1.2km $\times$ 1.2km rectangle grids, and the day into 10 minutes time intervals (suitable for most vehicles to traverse a grid). The simulator predicts cruising directions to neighboring grids (or stay in the same grids) for each driver every time step by our LSTM network (Fig.~\ref{Fig:PreferenceModeling}). Meanwhile, the simulator emulates the driver decision makings on accepting reposition recommendations by Eq.~\ref{eq:logitModel1}. The driver either visits the recommended location, or randomly visits one of the top preferred four grids (See Appendix A for preference prediction performance).
%
%
%


\subsubsection{Datasets}

We use a real-world taxi trajectory dataset collected in Chengdu, China, with a total of 14,865 taxis and 10,710,949 passenger trips within a month~\cite{lyu2019od}. 
For experiment efficiency, we sampled 500 taxis and 243,679 passenger demands. 
The road network of Chengdu is obtained from OpenStreetMap, spanning the city center within 30.62 - 30.69 latitude and 107.01 - 107.11 longitude.

\subsubsection{Baseline Techniques} 


We compare the following methods: 
\begin{itemize}[leftmargin=0.1in,itemsep = -0.02in,topsep=0.01in]
    \item \textit{No Reposition} does not reposition drivers but lets them cruise as estimated preferences in the simulator.
    \item \textit{Random} recommends idle vehicles to a random one of neighboring $3\times3$ grids in a random order.
    \item \textit{Demand Greedy} recommends vehicles to the current most demanding neighboring grid in a random order.
    \item \textit{Reward Greedy} suggests the most demanding and satisfactory neighboring grid to maximize reward by Eq.\ref{Eq:totalRewards}.
    \item \textit{Min Cost Flow}~\cite{ming2020effective} formulates vehicle repositioning as a min cost flow problem and minimizes repositioning cost while maximizing the total served demands. 
    \item \textit{Contextual DDQN}~\cite{liu2020context} uses DDQN to reposition idle vehicles and coordinates them by encoding prior driver decisions into subsequent driver observations.
    \item \textit{Proportional Reposition} defines its action as redistributing proportions of idle vehicles from a grid to its neighbors~\cite{mao2020dispatch}. This facilitates coordination among idle vehicles akin to joint action methods~\cite{liu2022deep}, yet maintains a relatively smaller action space.    
    \item \textit{By Collective Preference} considers collective preference of drivers by their historical visiting frequency~\cite{he2020spatio}, and thereby tends to recommend demanding grids that most drivers like to visit. 
    
\end{itemize}





\subsubsection{Metrics} 
We evaluate the performance of all the methods with the following metrics.
\begin{itemize}[leftmargin=0.1in,itemsep = -0.02in,topsep=0.01in]
    \item \textit{Total Driver Income (TDI)} is the total income from all the served demands in a day. Better supply-demand balance help to achieve a higher TDI.
    \item \textit{Reposition Income (RI)} is the total income earned from successful reposition recommendations. We use RI$/$TDI to measure how much is earned from repositions.
    \item \textit{Request Response Rate (RRR)} is the ratio between the number of orders got served to the total number of orders. 
    \item \textit{Acceptance Rate} is the ratio of accepted repositions to total recommendations, indicating driver satisfaction.
    \item \textit{Number of Repositions} is the number of accepted recommendations, excluding those for staying in the same grid. This assesses reposition cost for supply-demand balance.
\end{itemize}

\subsubsection{Implementation}
In i-Rebalance, we implemented the actor and critic networks with two fully connected layers (64 units, \textit{ReLU} activations), respectively.
Grid Agent and Vehicle Agents use the same architecture, normalizing outputs with softmax and sigmoid functions separately.
%
We used entropy loss and batch normalization, entropy-$\beta$ is set to 0.01 and the batch size is 10.
We set a decay factor $\gamma=0.98$ and train the network with \textit{Adam} optimizer~\cite{liu2020context}. 
In reward function (Eq.~\ref{Eq:totalRewards}), the weights $\alpha_B=2$, and $\alpha_P=1$ achieve the best performance. 
The models are trained on Intel Core i9-10940K CPU @3.30GHz, NVIDIA GeForce RTX 3090, and 32GB memory.

\subsection{Experimental Results}

\begin{table}[t]
\centering
\renewcommand\arraystretch{1.2}

\footnotesize
\begin{tabular}{|@{}>{\centering\arraybackslash} m{2.4cm}@{}|@{}>{\centering\arraybackslash}m{1.3cm}@{}|@{} >{\centering\arraybackslash}m{1.2cm}@{}|@{} >{\centering\arraybackslash}m{1.1cm}@{}|@{} >{\centering\arraybackslash}m{1.1cm}@{}|@{} >{\centering\arraybackslash}m{1.1cm}@{}|@{}C{1.2cm}@{}|@{}C{1.2cm}@{}|}
\hline
 Methods & Norm. TDI & RI/TDI & RRR & Accept. Rate & \#Repos. \\
\hline
No Reposition & 100.00\%  & N/A & 72.15\%  & N/A & 10,733    \\
\hline
Random & 103.02\%  & 88.85\%  & 73.60\%  & 48.39\% & 9,691     \\
\hline
Demand Greedy & 101.81\%  & 85.73\%  & 73.11\%  & 45.44\%  & 11,061    \\
\hline
Reward Greedy & 102.03\%  & 87.36\%  & 73.27\%  & 50.50\%  & 10828    \\
\hline
Min Cost Flow & 103.00\%   & 87.48\%   & 73.46\%   & 44.68\%   & 7,700     \\
\hline
Contextual DDQN & 102.46\%  & 87.44\%  & 73.33\%  & 46.07\%  & 10,866    \\
\hline
Proportional Repositioning & 102.59\%  & 88.20\%  & 73.60\%  & 22.96\%  & 6,749    \\
\hline
By Collective Preference & 109.00\%  & 95.28\%  & 77.29\%  & 66.48\%  & 6,803    \\
\hline
i-Rebalance & \textbf{109.97\%}  & \textbf{98.70\%}  & \textbf{78.35\%}  & \textbf{84.43\%}  & \textbf{6,625}    \\
\hline
\end{tabular}
\caption{Performance Comparison of Vehicle Reposition Methods. Total Driver Income (TDI) is normalized by dividing that of \textit{No Reposition}. 
}\label{tab:result}
\vspace{-0.2in}
\end{table}

\subsubsection{Performance Comparison}

Table~\ref{tab:result} summarizes the comparison results of our i-Rebalance against the baselines. Overall, i-Rebalance performs the best in all metrics, indicating that i-Rebalance can achieve a good balance of supply and demand, providing the most profitable, satisfactory, and cost-effective reposition recommendations.

\textit{No Reposition} has the lowest TDI by letting the drivers cruise by themselves without cooperating with each other.  \textit{Demand Greedy} has similar performance as it repositions vehicles to current most demanding grids, vehicles are easily gathered in some grids. 
\textit{Reward Greedy} performs better than \textit{Demand Greedy} by incorporating driver preference.
\textit{Random} has slightly higher TDI and ORR than \textit{Reward Greedy} because it disperses vehicles everywhere. 
\textit{Min Cost Flow} slightly worse than \textit{Random} as it minimizes  reposition cost. 
Overall, these non-learning-based methods have worse performance compared to the DRL methods, due to the naive strategies and short-term focus. 

%
Among the DRL methods, 
\textit{Contextual DDQN} has the worst performance as it 
coordinates vehicles by encoding early repositioning outcomes for state of later ones, which might not be sufficient. It also employs fixed rewards, less adaptable to an uncertain environment where drivers can reject repositions.
\textit{Proportional Repositioning} improves coordination via homogeneous joint action, yet ignores individual preference.
\textit{By Collective Preference} enhances performance with driver collective preference, achieving the highest acceptance rate (66.48\%) among the baselines. But it still fails to consider personalized preference.
In contrast, i-Rebalance incorporates personalized preference and coordinates drivers by learning their reposition order, achieving the highest driver income and acceptance rate.

\subsubsection{Ablation Study}
We conducted two ablation studies to evaluate 1)~effectiveness of learning recommendation order, and 2)~effectiveness of incorporating personalized preference with different designs of state and reward. 


\textit{Effectiveness of order learning}. We design 3 variations of i-Rebalance. \textit{Fixed-Order} employs a fixed order prioritizing drivers based on their preference distribution's correlation with supply-demand gaps (measured by Pearson correlation). \textit{Random-Order} employs a random order. \textit{Joint-Action} involves a single agent to jointly determine driver assignments to specific locations. 
Table~\ref{tab:ablation1} summarizes the comparison results. 
i-Rebalance excels by collaborating with the Vehicle Agent for the optimal order learning. \textit{Fix-Order} is slightly less effective than i-Rebalance but surpasses \textit{Random-Order} due to driver preference consideration. \textit{Joint-Action} performs worst due to its larger action space, challenging to find the optimal solutions.


\begin{table}[t]
\centering

\footnotesize
\begin{tabular}{|c|c|c|c|c|c|c|c}

\hline
Methods & Norm. TDI & RI/TDI & Accept. Rate  \\
\hline

Fixed-Order & 109.89\%  & 98.50\% & 83.40\%      \\
\hline
Random-Order & 109.77\%  & 98.44\% & 83.31\%      \\
\hline
Joint-Action & 103.15\%  & 90.18\% & 52.53\%       \\
\hline
i-Rebalance & 109.97\%  & 98.70\% & 84.43\%     \\
\hline

\end{tabular}

\caption{Performance of i-Rebalance variations with different reposition order.}
\label{tab:ablation1}
\vspace{-0.1in}
\end{table}

%


\textit{Effectiveness State and Reward Design for Personalization}. We create 3 i-Rebalance variations with different state and reward structures for driver preference in Vehicle Agent.
\textit{RNP-SNP} lacks both preference satisfaction reward and preference observation in Vehicle Agent, \textit{RNP-SP} has preference observation but not preference reward, and \textit{RP-SNP} incorporates preference reward but lacks preference observation. i-Rebalance combines both reward and observation, known as \textit{RP-SP}. For equitable comparison, we maintain the reposition order learned in i-Rebalance.
Table~\ref{tab:ablation2} summarizes comparison results. The four methods display similar total driver income (TDI) and minor variations in RI/TDI, yet show significant differences in acceptance rates. This shows that i-Rebalance can generate more satisfactory recommendations without sacrificing total driver income. Also, integrating driver preferences in state (\textit{RNP-SP}) is more effective than integrating that in reward (\textit{RP-SNP}).




\begin{table}[t]
\centering

\vspace{-0.1in}
\footnotesize
\begin{tabular}{|c|c|c|c|c|c|c|c|}
\hline
Methods & Norm. TDI & RI/TDI & Accept. Rate  \\
\hline
RNP-SNP & 109.37\%  & 95.69\% & 67.95\%     \\
\hline
RP-SNP & 109.40\%  & 95.75\% & 68.11\%      \\
\hline
RNP-SP & 109.73\%  & 98.29\% & 81.65\%      \\
\hline
RP-SP & 109.97\%  & 98.70\% & 84.43\%       \\
\hline
\end{tabular}

\caption{Performance of i-Rebalance variations with different state and reward design for Vehicle Agent. R: reward, S: state, P: preference, NP: no preference.}
\label{tab:ablation2}
\vspace{-0.2in}
\end{table}

\section{Conclusion}

In this paper, we present i-Rebalance, a personalized vehicle reposition technique to balance supply and demand and to improve driver acceptance of reposition recommendations.
%
i-Rebalance models driver cruising preferences with a lightweight LSTM network and estimates their decisions on accepting the recommendations through a user study. It sequentially repositions vehicles with dual DRL agents to learn the reposition order and personalized recommendation, respectively. 
%
%
Evaluation shows i-Rebalance 
greatly improves reposition acceptance of
drivers and total driver income
compared to all baselines.
Future work includes generalizing our sequential learning strategy to coordinate multi-agents for personalized decision making, and exploring personalized reposition incentives for drivers.  





\section{Acknowledgments}
This work was supported in part by the National Natural Science Foundation of China under Grant No. 62232004,  No.61972086, No. 62102082, No. 62072099, No.61932007; in part by the Jiangsu Natural Science Foundation of China under Grant No. BK20210203, BK20230024.

\bibliography{reference}
\newpage


\end{document}